\documentclass{article}
\usepackage{tikz}
\usetikzlibrary{arrows,decorations.markings}
\usepackage{amssymb,cite}
\usepackage{amsmath}

\usepackage{graphicx}
\usepackage[all]{xy}
\usepackage{color}
\usepackage{tcolorbox}
\usepackage[color=green!40]{todonotes}
\usepackage[color=green!40]{todonotes}
\usepackage{setspace}
\usepackage{amsthm,amsfonts,amssymb,epsfig,graphics,amsmath,amsbsy,subfigure}
\begin{document}
\def\mathbi#1{\textbf{\em #1}}
\newcommand{\intL}{\int\limits}
\newcommand{\half}{^\infty_0 }
\newcommand{\intR}{\int\limits_{\mathbb{R}} }
\newcommand{\intRR}{\int\limits_{\mathbb{R}^2} }
\newcommand{\RR}{\mathbb{R}}
\newcommand{\xx}{\mathbf{x}}
\newcommand{\yy}{\mathbf{y}}
\newcommand{\zz}{\mathbf{z}}
\newcommand{\ww}{\mathbf{w}}
\newcommand{\uu}{\mathbf{u}}
\newcommand{\xb}{\textbf{\textit{x}}}
\newcommand{\yb}{\textbf{\textit{y}}}
\newcommand{\zb}{\textbf{\textit{z}}}
\newcommand{\pb}{\textbf{\textit{p}}}
\newcommand{\ttheta}{\boldsymbol{\theta}}
\newcommand{\pphi}{\boldsymbol{\phi}}
\newcommand{\oomega}{\boldsymbol{\omega}}
\newcommand{\eeta}{\boldsymbol{\eta}}
\newcommand{\zzeta}{\boldsymbol{\zeta}}
\newcommand{\ttau}{\boldsymbol{\tau}}
\newcommand{\xxi}{\boldsymbol{\xi}}
\newcommand{\ReLU}{\sigma}
\newcommand{\rmd}[1]{\mathrm d#1}
\newtheorem{thm}{Theorem}
\newtheorem{defi}[thm]{Definition}
\newtheorem{rmk}[thm]{Remark}
\newtheorem{cor}[thm]{Corollary}
\newtheorem{lem}[thm]{Lemma}
\newtheorem{prop}[thm]{Proposition}

\title{Error estimate for a universal function approximator of ReLU network with a local connection}
\author{Jae-Mo Kang$^a$ and Sunghwan Moon$^b$}

\date{${}^a$Department of Artificial Intelligence, ${}^b$Department of Mathematics,\\
Kyungpook National University, Daegu 41566, Republic of Korea \\
{\tt sunghwan.moon@knu.ac.kr}}

\maketitle
\begin{abstract}
Neural networks have shown high successful performance in a wide range of tasks, but further studies are needed to improve its performance.
We analyze the approximation error of the specific neural network architecture with a local connection and higher application than one with the full connection because the local-connected network can be used to explain diverse neural networks such as CNNs.
Our error estimate depends on two parameters: one controlling the depth of the hidden layer, and the other, the width of the hidden layers.\\

\end{abstract}

%
%
%
\section{Introduction}
Neural networking has achieved outstanding performance in a wide range of areas such as computer vision and natural language processing (see the review article \cite{lecunbh15} and recent book \cite{goodfellowbc16}  for more background), and the question of why they work so well have naturally attracted much attention.
\cite{barron94,cybenko89,funahashi89,horniksw89} have already answered the question about a network with a single hidden layer in detail. A network with a single hidden layer can approximate any continuous function with compact support to arbitrary accuracy when the width goes to infinity. This result is referred to as \textit{the universal approximation theorem}.

This classical universal approximation theorems focused on the width increasing to infinity, namely the fat network. However, the recent tremendous success of the neural network originates from the larger and deeper network structure. The literature has reported the theoretical support for the deep neural network  \cite{bauerk19,cohenss16,eldans16,lupwhw17,linj18,mhaskarp16,ohnk19,telgarsky16,hanins18,hanin19}.

Many literature error estimates for approximators exist \cite{bauerk19,hanin19,montanellid19,yarotsky17}. Yarotsky provides the $L^\infty$-error of approximation of functions belonging to a certain Sobolev space with a fully connected ReLU network \cite{yarotsky17}. Hanin gives quantitative depth estimates on only the length of hidden layers for the rate of approximation of any continuous function by fully connected ReLU networks with a bounded width \cite{hanin19}. Also, Bauer and Kohler show that least squares estimates based on multilayer feedforward neural networks allow to circumvent the curse of dimensionality in nonparametric regression in \cite{bauerk19}.

In recent practical applications of deep learning, special network structures involving several local connections, such as CNNs, were also widely adopted. 
Many works have tried to answer the question for the case of feedforward neural networks with full connection; however, not many literatures  have been answered the question yet for neural networks with a local connection \cite{zhou20nn,zhou20uni}.
 The main contributions of this paper are summarized as follows:
\begin{itemize}
\item We analyze the approximation error of the specific neural network architecture with a local connection.
\item Notably, the analytical results derived in this paper show that the locally-connected networks (CNNs) are universal function approximators. 
\item Through our analysis, we provide new, interesting, useful, and helpful engineering insights to understand the approximation performance of the network architecture with local connections, and why such network architecture is working well.
\end{itemize}

\subsection{Notations}
Let $\mathbb Z_k=\{-k,-k+1,\cdots ,k\}$, $g_1(x)=\ReLU(1-|x|)$ where $ \ReLU(x)=\max\{x,0\}$ and $g_n(\xx)=\prod_{j=1}^ng_1(x_j),\xx=(x_1,x_2\cdots,x_n)\in \RR^n$ (see Figure \ref{fig:h}).
 \begin{figure*}
\centering
\includegraphics[width=0.25\textwidth] {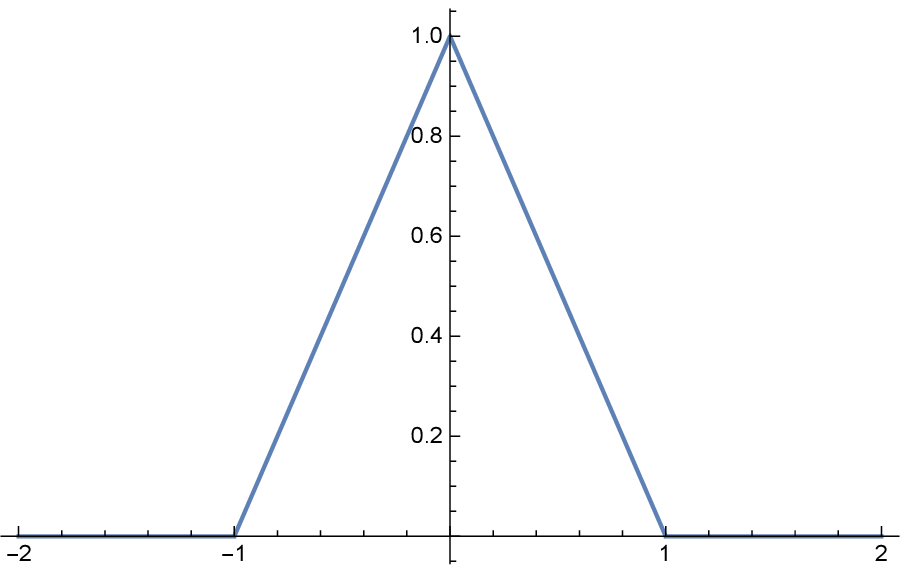}
\includegraphics[width=0.25\textwidth] {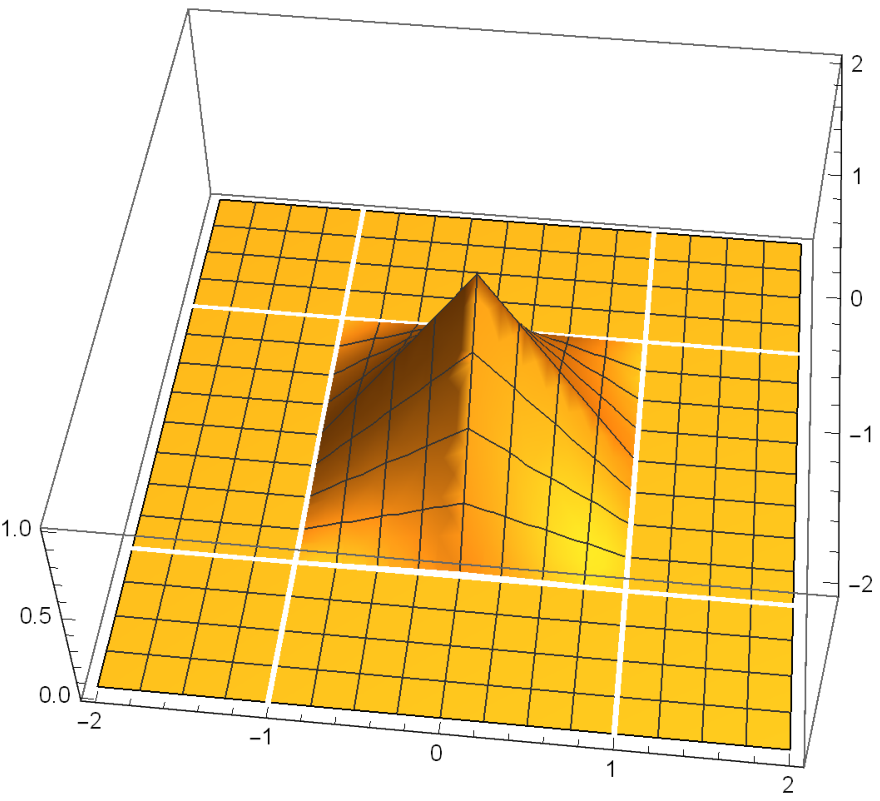}
\includegraphics[width=0.2\textwidth] {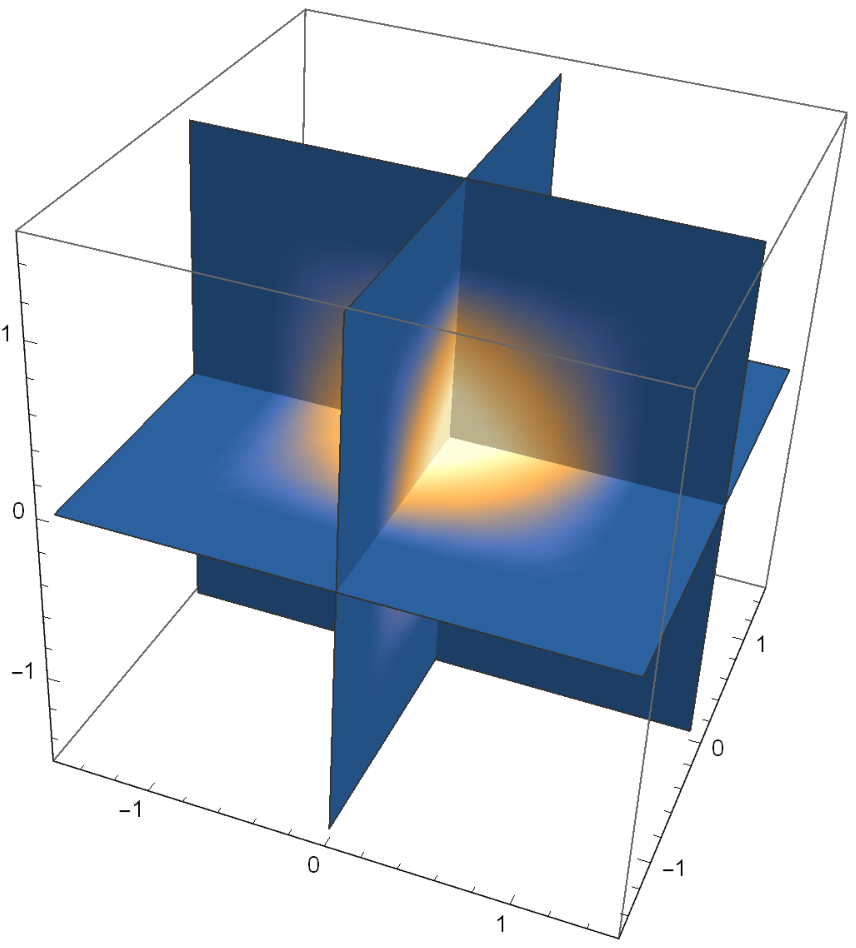}
\includegraphics[width=0.03\textwidth] {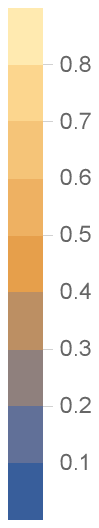}
\caption{ Left: $g_1(x)$ Middle: $g_2(\xx)$ Right: $g_3(\xx)$}
\label{fig:h}
\end{figure*}
The network architecture $(L,\mathbf n)$ consists of several hidden layers $L\in\mathbb N$ and a width vector $\mathbf n=(n_0,n_1,\cdots,n_{L+1})\in\mathbb N^{L+2}.$ A neural network with network architecture $(L,\mathbf n)$ is then any function of the form
\begin{equation}\label{eq:netwrokfunction}
\begin{array}{ll}
f_{(L,\mathbf n)}:\RR^{n_0}\to \RR^{n_{L+1}},\\
f_{(L,\mathbf n)}(\xx)=W_{L+1} \circ \ReLU \circ W_L \circ \ReLU \circ \cdots \circ  W_2 \circ \ReLU \circ  W_1(\xx)
\end{array}
\end{equation}
where $\ReLU(\xx)=(\ReLU(x_1),\ReLU(x_2),\cdots,\ReLU(x_n))$, $W_i=(w^{(i)}_{lm})\in\RR^{n_i}\times\RR^{n_{i-1}+1}$ and $W(\xx):\RR^{n_{i-1}}\to \RR^{n_i}$ is defined by
$$
\begin{array}{ll}
\displaystyle W_i(\xx)=\left(\sum_{m=1}^{n_{i-1}}w_{1m}x_{m}+w_{1,n_{i-1}+1},\sum_{m=1}^{n_{i-1}}w_{2m}x_{m}+w_{2,n_{i-1}+1},\right.\left.\cdots,\sum_{m=1}^{n_{i-1}}w_{n_im}x_{m}+w_{n_i,n_{i-1}+1}\right).
\end{array}
$$

Here we set $n_0=n$ and $n_{L+1}=1$.
To this end, we define the space of network functions with the given network architecture
$$
\mathcal F(L,\mathbf n):=\{f \mbox{ of the form \eqref{eq:netwrokfunction}}\}.
$$
\section{Main Results}

\begin{thm}\label{thm:error}
Let $f\in C^1(\RR^n)$ with $\operatorname{supp}f\subset[-1,1]^n$ and
$$
(f*g_n)_k(\xx)=\sum^{}_{\mathbf{i}\in\mathbb Z_k^n}f(\mathbf{i}k^{-1})g_{n}(k\xx-\mathbf i).
$$
Then we have
$$
|f(\xx)-(f*g_n)_k(\xx)|\le \max_{j=1,2,\cdots,n}||\partial_{x_j}f(\cdot)||_\infty n (2k)^{-1}.
$$
\end{thm}
Appendix \ref{sec:proof} presents the proof.
Let us denote by $||f||_\infty=\max\{ |f(\xx)|:\xx\in[0,1]^n\}$ and  $\lceil x\rceil=[x]+1$, where $[x]$ is a Gauss symbol.


\begin{thm}
Let $ f\in C^1(\RR^n)$ with $\operatorname{supp}f\subset[0,1]^n$ and $m,k\in\mathbb N$.
For $n\ge2$, there is a network function $f_{(L,\mathbf n)}\in \mathcal F((m+5)\lceil \log_2 n\rceil +3,(n,3n(2k+1), n(2k+1),6n(2k+1)^n,\cdots,6n(2k+1)^n ,(2k+1)^n,1)) $   such that
\begin{equation}\label{eq:estimate}
\begin{array}{ll}
\displaystyle \sup_{\xx\in[0,1]^n}|f(\xx)-f_{(L,\mathbf n)}(\xx)|\\
\le \displaystyle  ||f||_\infty3^n2^{-m}(2k+1)+ \max_{j=1,2,\cdots,n}||\partial_{x_j}f(\xx)||_\infty n (2k)^{-1},
\end{array}
\end{equation}
where $L=(m+5)\lceil \log_2 n\rceil +3$ and $\mathbf n=(n,3n(2k+1), n(2k+1),6n(2k+1)^n,\cdots,6n(2k+1)^n ,(2k+1)^n,1)$.
For $n=1$, there is a network function $f_{(2,(1,3(2k+1),2k+1,1))}\in \mathcal F(2,(1,3(2k+1),2k+1,1)) $ such that
\begin{equation*}
\sup_{\xx\in[0,1]^n}|f(\xx)-f_{(2,(1, 3(2k+1),2k+1,1))}(\xx)|\le ||f'||_\infty (2k)^{-1}.
\end{equation*}
\end{thm}
This \eqref{eq:estimate} can be written as
\begin{equation*}\label{eq:estimateO}\sup_{\xx\in[0,1]^n}|f(\xx)-f_{(L,\mathbf n)}(\xx)|=O(2^{-L/\lceil \log_2 n\rceil}k)+O((k^n)^{-1/n}).\end{equation*}
Here, $L$ is the  number of hidden layers, and $6n(2k+1)^n$ is the number of nodes of hidden layers.
If $k^2\approx 2^{L/\lceil \log_2 n\rceil}$, then
\begin{equation*}\label{eq:estimate0O}\sup_{\xx\in[0,1]^n}|f(\xx)-f_{(L,\mathbf n)}(\xx)|=O(2^{-L/(2\lceil \log_2 n\rceil)}).\end{equation*}

Theorem 2 shows that the approximation error (or accuracy) of the constructed network architecture depends on two factors: the number $L$ of the hidden layers (i.e., depth) of the network, and the parameter $k$ that is related to the maximum number $6n(2k+1)^n$ among the nodes of the hidden layers. Specifically, as shown in \eqref{eq:estimate}, the approximation error is proportional to these two factors and, more precisely, the higher increase of the number $L$ of the hidden layers than the number of nodes of the hidden layers is more efficient to achieve sufficiently high accuracy for function approximation. Notably, Theorem 2 implies that the specific network architecture with local connections, which should provide higher use than the one with the full connection because the locally-connected one is 1D CNN, is a universal function approximator.

%
\begin{rmk}
We can have a sharper estimate using $\operatorname{supp}f\subset[0,1]^n$.
When computing $(f*g_n)_k$, we need only half, i.e.
$$
(f*g_n)_k(\xx)=\sum^{}_{\mathbf{i}\in\mathbb Z_k^n}f(\mathbf{i}k^{-1})g_{n}(k\xx-\mathbf i)=\sum^{}_{\mathbf{i}\in(\mathbb N_k^{*})^n}f(\mathbf{i}k^{-1})g_{n}(k\xx-\mathbf i),
$$
where $\mathbb N_k^*=\{0,1,\cdots ,k\}$.
Thus, there is a network function of $f_{(L,\mathbf n)}\in \mathcal F((m+5)\lceil \log_2 n\rceil +3,(n,3n(k+1), n(k+1),6n(k+1)^n,\cdots,6n(k+1)^n ,(k+1)^n,1)) $ with
$$
\begin{array}{ll}
\displaystyle\sup_{\xx\in[0,1]^n}|f(\xx)-f_{(L,\mathbf n)}(\xx)|\le \displaystyle ||f||_\infty3^n2^{-m}(k+1)+ \max_{j=1,2,\cdots,n}||\partial_{x_j}f(\xx)||_\infty n (2k)^{-1}.
\end{array}
$$
\end{rmk}
\section{Proof}

\subsection{1-dimensional input}
To convey our idea clearer, we first construct a ReLU network with $n_{}=1$ and $2$ hidden layers with a width of at most $3(2k+1)$.
We then construct a ReLU network with a general input dimension $n$.

Let $W_{1,temp}:\RR^2\to \RR^3$ and $W_{2,temp}:\RR^3\to\RR$ with
$$
W_{1,temp}(x)=\left(x-1,x,x+1\right)
$$
and 
$$
W_{2,temp}(x_1,x_2,x_3)=x_1-2x_2+x_3.
$$
Then the function represented by $W_{2,temp}\circ\ReLU\circ W_{1,temp}$ is
$$
\begin{array}{ll}
W_{2,temp}\circ\ReLU\circ W_{1,temp}(x)
=\ReLU(x-1)-2\ReLU(x)+\ReLU(x+1)
=\left\{\begin{array}{ll}0\quad&\mbox{if}\quad |x|>1,\\
1-|x|\quad&\mbox{if}\quad -1<x<1,
\end{array}\right.
\end{array}
$$
which is equal to $g_1(x)$.

By Theorem \ref{thm:error}, we have
$$
\sup_{-{1}\le x\le{1}}\left|f(x)- \sum^k_{i=-k}f(ik^{-1})g_1(kx-i )\right|<||f'||_\infty (2k)^{-1}.
$$

Now we contruct a feedforward neural network $(L,\mathbf n)$ with ReLU activations, input and output dimensions 1, and 2 hidden layer width at most $3(2k+1)$ whose repesent $(f*g_n)_k$: Then
 $$
\sup_{-1\le x\le1} \left|f(x)- f_{(L,\mathbf n)}(x)\right|< ||f'||_\infty (2k)^{-1}.
 $$
 Let $W_{1}:\RR\to \RR^{(2k+1)\times3}$, $W_{2}:\RR^{(2k+1)\times3}\to \RR^{2k+1},$ and $W_{3}:\RR^{2k+1}\to \RR$ be defined by
 $$
 \begin{array}{ll}
W_{1}(x)=\left(\begin{array}{ll}\displaystyle kx+k-1,kx+k,kx+k+1\\
\displaystyle kx+k-2,kx+k-1,kx+k\\
\qquad\vdots\\
\displaystyle kx-i-1,kx-i,kx-i+1\\
\qquad\vdots\\
\displaystyle kx-k-1,kx-k,kx-k+1\\
\end{array}\right),\\
 W_{2}(\xx)=\left(\begin{array}{ll}\displaystyle x_{11}-2x_{12}+x_{13}\\
\displaystyle x_{21}-2x_{22}+x_{23}\\
\qquad\vdots\\
\displaystyle x_{i1}-2x_{i2}+x_{i3}\\
\qquad\vdots\\
\displaystyle x_{k1}-2x_{k2}+x_{k3}\\
\end{array}\right) 
\mbox{and}\quad
W_{3}(\xx)=\left(\begin{array}{ll}
\displaystyle \sum_{i=-k}^kf(ik^{-1})x_i\\
\end{array}\right)\mbox{(see Figure \ref{fig:n=1})}.
\end{array}
$$
Here, we express the vector in the $\RR^{(2k+1)\times3}$ to the $(2k+1)\times 3$ matrix for easy conveyance of the idea.
Then we have
$$
W_2\circ \ReLU\circ  W_1(x)=\left(\begin{array}{c}
 \displaystyle g_1(kx+k)\\
\displaystyle g_1(kx+k-1)\\
\vdots\\
\displaystyle g_1(kx-i)\\
\vdots\\
\displaystyle g_1(kx-k)\\
\end{array}\right)
$$
and thus
$$
\begin{array}{ll}W_3\circ\ReLU\circ W_2\circ \ReLU\circ W_1(x)
\displaystyle =\sum^k_{i=-k}f(ik^{-1})g_1(kx-i),
\end{array}
$$
because for a non-negative function $h:\RR\to \RR$, $\ReLU(h)=h$.
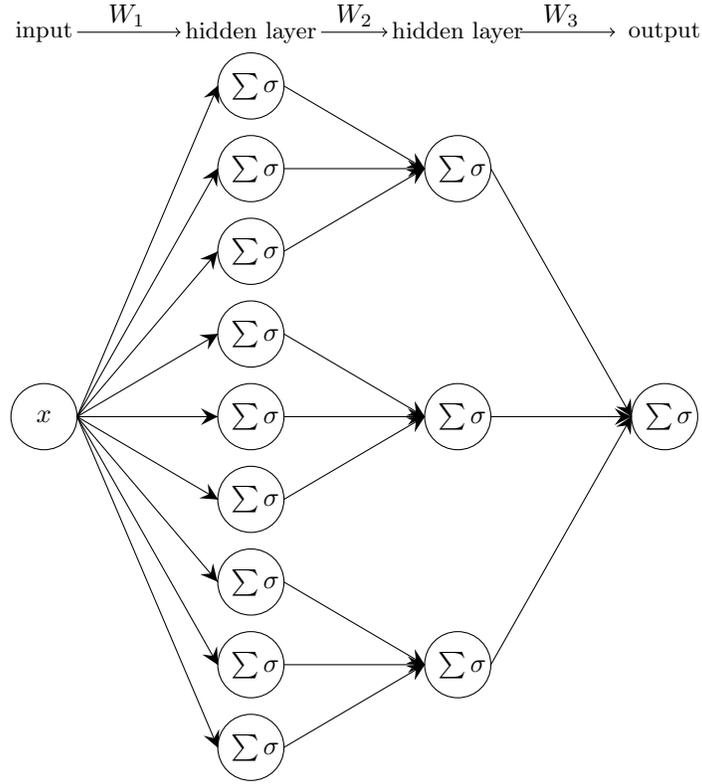
\begin{figure}
\begin{center}
\begin{tikzpicture}[scale=0.55]
\tikzset{myptr/.style={decoration={markings,mark=at position 1 with %
    {\arrow[scale=2,>=stealth]{>}}},postaction={decorate}}}
   \draw (5,15) circle (0.8);
  \draw (5,13) circle (0.8);
  \draw (5,11) circle (0.8);
  \draw (5,9) circle (0.8);
  \draw (5,5) circle (0.8);
  \draw (5,7) circle (0.8);
   \draw (5,3) circle (0.8);
   \draw (5,1) circle (0.8);
    \draw (5,-1) circle (0.8);
  \draw (10,13) circle (0.8);
  \draw (10,7) circle (0.8);
   \draw (10,1) circle (0.8);
    \draw (15,7) circle (0.8);
  \draw[myptr] (5.8,15)--(9.2,13);
  \draw[myptr] (5.8,13)--(9.2,13);
  \draw[myptr] (5.8,11)--(9.2,13);
  \draw[myptr] (5.8,9)--(9.2,7);
  \draw[myptr] (5.8,7)--(9.2,7);
  \draw[myptr] (5.8,5)--(9.2,7);
  \draw[myptr] (5.8,3)--(9.2,1);
    \draw[myptr] (5.8,1)--(9.2,1);
      \draw[myptr] (5.8,-1)--(9.2,1);
  \draw[myptr] (10.8,7)--(14.2,7);
  \draw[myptr] (10.8,13)--(14.2,7);
    \draw[myptr] (10.8,1)--(14.2,7);
   \draw[->] (11.5,16.3)--(13.8,16.3);
   \node at (0,7) { $x$};
  \node at (5,-1) {{\small}  $\sum \ReLU$};
    \node at (5,1) {{\small}  $\sum \ReLU$};
      \node at (5,3) {{\small}  $\sum \ReLU$};
 \node at (5,15) {{\small}  $\sum \ReLU$};
 \node at (5,13) {{\small}  $\sum \ReLU$};
 \node at (10,13) {{\small}  $\sum \ReLU$};
 \node at (10,7) {{\small}  $\sum \ReLU$};
    \node at (5,11) {{\small}  $\sum \ReLU$};
  \node at (5,9) {{\small}  $\sum \ReLU$};
  \node at (5,7) {{\small}  $\sum \ReLU$};
  \node at (5,5) {{\small}  $\sum \ReLU$};
  \node at (10,1) {{\small}  $\sum \ReLU$};
  \node at (15,7) {{\small}  $\sum \ReLU$};
  \draw (0,7) circle (0.8) ;
  \draw[myptr] (0.8,7) -- (4.2,15);
  \draw[myptr] (0.8,7) -- (4.2,13);
  \draw[myptr] (0.8,7) -- (4.2,11);
  \draw[myptr] (0.8,7) -- (4.2,9);
  \draw[myptr] (0.8,7) -- (4.2,7);
  \draw[myptr] (0.8,7) -- (4.2,5);
  \draw[myptr] (0.8,7) -- (4.2,-1);
  \draw[myptr] (0.8,7) -- (4.2,1);
  \draw[myptr] (0.8,7) -- (4.2,3);
  \draw[->] (0.8,16.3)--(3.3,16.3);
  \draw[->] (6.7,16.3)--(8.3,16.3);
  \node at (5,16.3) {\small hidden layer};
  \node at (10,16.3) {\small hidden layer};
   \node at (15,16.3) {\small output};
   \node at (0,16.3) {\small input};
   \node at (2,16.7) {$W_1$};
    \node at (7.5,16.7) {$W_2$};
   \node at (12.5,16.7) {$W_3$};
\end{tikzpicture}
\end{center}
\caption{neural network with $n=1$ and $k=1$}\label{fig:n=1}
\end{figure}
\subsection{General-dimensional input}
As in a 1-dimensional input case, we first construct the approximate identity $g_n(\xx)$ for $n$ dimensions.
Let $W_{1,temp}(\xx):\RR^n\to \RR^{n \times 3}$ and $W_{2,temp}(\xx): \RR^{n\times 3}\to\RR^{n}$ be defined by
$$
W_{1,temp}(\xx)=\left(\begin{array}{cc}x_1-1,x_1,x_1+1\\
x_2-1,x_2,x_2+1\\
\vdots\\
x_n-1,x_n,x_n+1\end{array}\right)
$$
and
$$
 W_{2,temp}(\xx)=\left(\begin{array}{cc}x_{1,1}-2x_{1,2}+x_{1,3}\\
x_{2,1}-2x_{2,2}+x_{2,3}\\
\vdots\\
x_{n,1}-2x_{n,2}+x_{n,3}
\end{array} \right).
$$
Then we have
$$
 W_{2,temp}\circ \ReLU\circ W_{1,temp}(\xx)=\left(\begin{array}c g_1(x_1)\\g_1(x_2)\\\cdots\\g_1(x_n)\end{array}\right).
$$

Now, we construct the function $g_n$, the multiplication of $g_1$ that is presented in \cite{schmidthieber17,telgarsky16,yarotsky17}.
\begin{lem}[Lemma A3 in \cite{schmidthieber17}]\label{lem:mul}
There exists a network $\operatorname{Mult}_m^r\in \mathcal F((m+5)\lceil \log_2r \rceil,(r,6r,6r,\cdots,6r,1))$ such that
$$
\left|\operatorname{Mult}_m^r(\xx)-\displaystyle\prod_{j=1}^rx_j\right|\le 3^r2^{-m},
$$
\end{lem}
for all $\xx=(x_1,x_2,\cdots,x_r)\in[0,1]^r.$

Applying Lemma \ref{lem:mul}, we have a network $$\operatorname{AppId}_{m,temp}^n\in \mathcal F((m+5)\lceil \log_2n \rceil+2,(n,3n, n,6n,6n,\cdots,6n,1))$$ such that
$$
\displaystyle\left|\operatorname{AppId}_{m,temp}^n(\xx)-\prod_{j=1}^ng_1(x_j)\right|\le 3^n2^{-m},
$$
for all $\xx=(x_1,x_2,\cdots,x_n)\in[0,1]^n.$
Now, we construct a feedforward neural network $(L,\mathbf n)$ with ReLU activations, input dimension $n$, and output dimension 1.
Let $W_1(\xx):\RR^n\to \RR^{(2k+1)\times n \times 3}$ and $W_2(\xx): \RR^{(2k+1)\times n\times 3}\to\RR^{n\times (2k+1)}$ be defined by
$$
\begin{array}{ll}
W_1(\xx)=\left(\begin{array}{cc}\left(\begin{array}{cc}kx_1+k-1,kx_1+k,kx_1+k+1\\
kx_1+k-2,kx_1+k-1,kx_1+k\\
\vdots\\
kx_1-k-1,kx_1-k,kx_1-k+1\end{array}\right)\\
\left(\begin{array}{cc}kx_2+k-1,kx_2+k,kx_2+k+1\\
\vdots\\
kx_2-k-1,kx_2-k,kx_2-k+1\end{array}\right)\\
\vdots\\
\left(\begin{array}{cc}kx_n+k-1,kx_n+k,kx_n+k+1\\
\vdots\\
kx_n-k-1,kx_n-k,kx_n-k+1\end{array}\right)
\end{array}\right)
\quad\mbox{and}\\
\displaystyle W_2(\xx)=\left(\begin{array}{cccc}x_{1,1}-2x_{1,2}+x_{1,3},&\cdots,&x_{n,1}-2x_{n,2}+x_{n,3}\\
x_{n+1,1}-2x_{n+1,2}+x_{n+1,3},&\cdots,&x_{2n,1}-2x_{2n,2}+x_{2n,3}\\
\vdots&\vdots&\vdots\\
x_{n^2-n+1,1}-2x_{n^2-1,2}+x_{n^2-1,3},&\cdots,&x_{n^2,1}-2x_{n^2,2}+x_{n^2,3}\end{array}
 \right)\\
\mbox{(see Figure \ref{fig:n=2})}.
 \end{array}
$$
Then we have
$$
\begin{array}{ll}
 W_2\circ \ReLU\circ W_1(\xx)\\
 =\left(\begin{array}{llll}g_1(x_1+k)&g_1(x_1+k-1)&\cdots&g_1(x_1-k)\\
g_1(x_2+k)&g_1(x_2+k-1)&\cdots&g_1(x_2-k)\\
&\vdots&&\\
g_1(x_n+k)&g_1(x_n+k-1)&\cdots&g_1(x_n-k) \end{array}\right).
\end{array}
$$

\begin{figure}
\begin{center}
\begin{tikzpicture}[scale=0.55]
\tikzset{myptr/.style={decoration={markings,mark=at position 1 with %
    {\arrow[scale=2,>=stealth]{>}}},postaction={decorate}}}
  \draw (0,-5) circle (0.8) ;
  \draw (0,13) circle (0.8) ;
  \node at (0,13) {$x_1$};

  \node at (0,-5) {$x_2$};
  \draw (5,21) circle (0.8) ;
  \draw (5,19) circle (0.8) ;
  \draw (5,17) circle (0.8) ;
  \draw (5,15) circle (0.8) ;
  \draw (5,13) circle (0.8) ;
  \draw (5,11) circle (0.8) ;
  \draw (5,9) circle (0.8) ;
  \draw (5,7) circle (0.8) ;
  \draw (5,5) circle (0.8) ;
  \draw (5,-3) circle (0.8) ;
  \draw (5,-5) circle (0.8) ;
  \draw (5,1) circle (0.8) ;
  \draw (5,-1) circle (0.8) ;
  \draw (5,3) circle (0.8) ;
  \draw (5,-7) circle (0.8) ;
  \draw (5,-9) circle (0.8) ;
  \draw (5,-11) circle (0.8) ;
  \draw (5,-13) circle (0.8) ;

  \node at (5,15) {{\small}  $ \sum\ReLU$};
  \node at (5,9)  {{\small}  $ \sum\ReLU$};
  \node at (5,13) {{\small}  $ \sum\ReLU$};
  \node at (5,3)  {{\small}  $ \sum\ReLU$};
  \node at (5,-5) {{\small}  $ \sum\ReLU$};
  \node at (5,-7) {{\small}  $ \sum\ReLU$};
  \node at (5,5)  {{\small}  $ \sum\ReLU$};
  \node at (5,-3) {{\small}  $ \sum\ReLU$};
  \node at (5,1)  {{\small}  $ \sum\ReLU$};
  \node at (5,-1) {{\small}  $ \sum\ReLU$};
  \node at (5,11) {{\small}  $ \sum\ReLU$};
  \node at (5,7)  {{\small}  $ \sum\ReLU$};
  \node at (5,17) {{\small}  $ \sum\ReLU$};
  \node at (5,19)  {{\small}  $ \sum\ReLU$};
  \node at (5,21) {{\small}  $ \sum\ReLU$};
  \node at (5,-9)  {{\small}  $ \sum\ReLU$};
  \node at (5,-11) {{\small}  $ \sum\ReLU$};
  \node at (5,-13) {{\small}  $ \sum\ReLU$};
  \draw[myptr] (0.8,13) -- (4.2,15);
  \draw[myptr] (0.8,13) -- (4.2,13);
  \draw[myptr] (0.8,13) -- (4.2,11);
  \draw[myptr] (0.8,13) -- (4.2,7);
  \draw[myptr] (0.8,13) -- (4.2,9);
  \draw[myptr] (0.8,13) -- (4.2,5);
  \draw[myptr] (0.8,13) -- (4.2,17);
  \draw[myptr] (0.8,13) -- (4.2,19);
  \draw[myptr] (0.8,13) -- (4.2,21);
  \draw[myptr] (0.8,-5) -- (4.2,1);
  \draw[myptr] (0.8,-5) -- (4.2,3);
  \draw[myptr] (0.8,-5) -- (4.2,-3);
  \draw[myptr] (0.8,-5) -- (4.2,-1);
  \draw[myptr] (0.8,-5) -- (4.2,-5);
  \draw[myptr] (0.8,-5) -- (4.2,-7);
  \draw[myptr] (0.8,-5) -- (4.2,-11);
  \draw[myptr] (0.8,-5) -- (4.2,-9);
  \draw[myptr] (0.8,-5) -- (4.2,-13);
  \draw (10,13) circle (0.8) ;
  \draw (10,7) circle (0.8) ;

  \draw (10,-5) circle (0.8) ;
  \draw (10,1) circle (0.8) ;
  \draw (10,19) circle (0.8) ;
  \draw (10,-11) circle (0.8) ;
  \draw (15,20) circle (0.9) ;
  \draw (15,16) circle (0.9) ;
  \draw (15,12) circle (0.9) ;
  \draw (15,8) circle (0.9) ;
  \draw (15,4) circle (0.9) ;
  \draw (15,0) circle (0.9) ;
  \draw (15,-4) circle (0.9) ;
  \draw (15,-8) circle (0.9) ;
  \draw (15,-12) circle (0.9) ;

  \node at (10,13) {{\small}  $ \sum\ReLU$};
  \node at (10,7)  {{\small}  $ \sum\ReLU$};
  \node at (10,1)  {{\small}  $ \sum\ReLU$};
  \node at (10,-5) {{\small}  $ \sum\ReLU$};
  \node at (10,19)  {{\small}  $ \sum\ReLU$};
  \node at (10,-11) {{\small}  $ \sum\ReLU$};
  \node at (15,20) {{\small}  $ \sum\ReLU\atop{(\operatorname{AppId}^2_{m,\mathbf i})}$};
  \node at (15,16) {{\small}  $ \sum\ReLU\atop{(\operatorname{AppId}^2_{m,\mathbf i})}$};
  \node at (15,12)  {{\small}  $ \sum\ReLU\atop{(\operatorname{AppId}^2_{m,\mathbf i})}$};
  \node at (15,8)  {{\small}  $ \sum\ReLU\atop{(\operatorname{AppId}^2_{m,\mathbf i})}$};
  \node at (15,4) {{\small}  $ \sum\ReLU\atop{(\operatorname{AppId}^2_{m,\mathbf i})}$};
  \node at (15,0) {{\small}  $ \sum\ReLU\atop{(\operatorname{AppId}^2_{m,\mathbf i})}$};
  \node at (15,-4)  {{\small}  $ \sum\ReLU \atop{(\operatorname{AppId}^2_{m,\mathbf i})}$};
  \node at (15,-8)  {{\small}  $ \sum\ReLU\atop{(\operatorname{AppId}^2_{m,\mathbf i})}$};
  \node at (15,-12) {{\small}  $ \sum\ReLU\atop{(\operatorname{AppId}^2_{m,\mathbf i})}$};
  \draw[myptr] (5.8,15)--(9.2,13);
  \draw[myptr] (5.8,13)--(9.2,13);
  \draw[myptr] (5.8,11)--(9.2,13);

  \draw[myptr] (5.8,21)--(9.2,19);
  \draw[myptr] (5.8,19)--(9.2,19);
  \draw[myptr] (5.8,17)--(9.2,19);


  \draw[myptr] (5.8,7)--(9.2,7);
  \draw[myptr] (5.8,9)--(9.2,7);
  \draw[myptr] (5.8,5)--(9.2,7);



  \draw[myptr] (5.8,-1)--(9.2,1);
  \draw[myptr] (5.8,3)--(9.2,1);
  \draw[myptr] (5.8,1)--(9.2,1);

  \draw[myptr] (5.8,-7)--(9.2,-5);
  \draw[myptr] (5.8,-5)--(9.2,-5);
  \draw[myptr] (5.8,-3)--(9.2,-5);


  \draw[myptr] (5.8,-9)--(9.2,-11);
  \draw[myptr] (5.8,-11)--(9.2,-11);
  \draw[myptr] (5.8,-13)--(9.2,-11);

  \node at (12.5,13) { $\cdots$};
  \node at (12.5,7)  { $\cdots$};
  \node at (12.5,1)  { $\cdots$};
  \node at (12.5,-5) { $\cdots$};
  \node at (12.5,-11) { $\cdots$};
  \node at (12.5,19) { $\cdots$};
  \draw (20,4) circle (0.8) ;
  \node at (20,4) {{\small} $ \sum\ReLU$};
  \draw[myptr] (15.8,20)--(19.2,4);
  \draw[myptr] (15.8,16)--(19.2,4);
  \draw[myptr] (15.8,12)--(19.2,4);
  \draw[myptr] (15.8,8)--(19.2,4);
  \draw[myptr] (15.8,4)--(19.2,4);
  \draw[myptr] (15.8,0)--(19.2,4);
  \draw[myptr] (15.8,-4)--(19.2,4);
  \draw[myptr] (15.8,-8)--(19.2,4);
  \draw[myptr] (15.8,-12)--(19.2,4);
  \node at (5,22.3) {\small hidden layer};
  \node at (10,22.3) {\small hidden layer};
  \node at (15,22.3) {\small hidden layer};
  \node at (12.5,22.3) { $\cdots$};
  \node at (20,22.3) {\small output};
  \node at (0,22.3) {\small input};
  \draw[->] (0.8,22.3)--(3.3,22.3);
  \draw[->] (6.7,22.3)--(8.3,22.3);
  \draw[->] (16.7,22.3)--(18.8,22.3);
  \node at (2,22.7) {$W_1$};
  \node at (7.5,22.7) {$W_2$};
  \node at (17.7,22.7) {$W_{Last}$};
\end{tikzpicture}
\end{center}
\caption{neural network with $n=2$ and $k=1$}
\label{fig:n=2}
\end{figure}
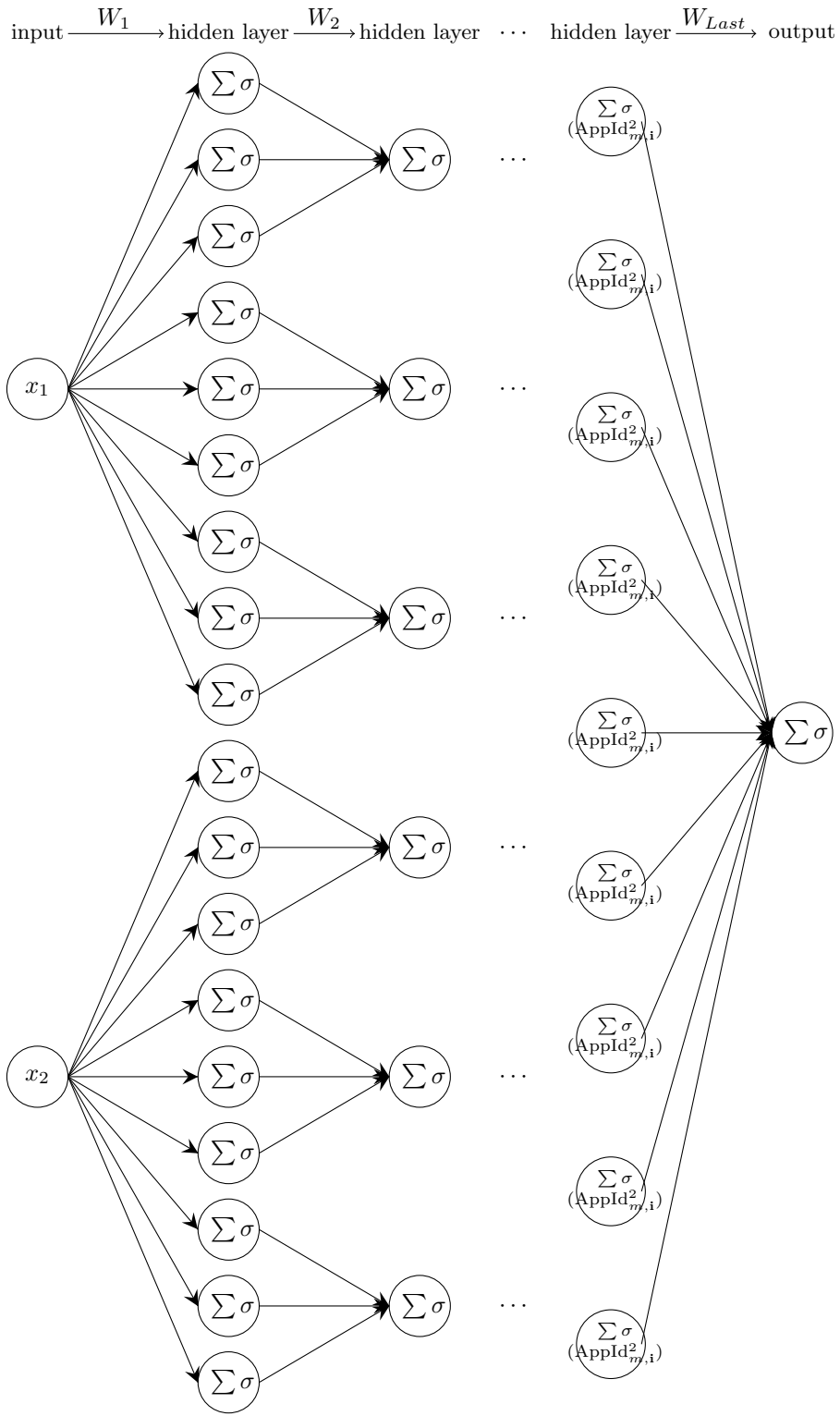

Now we apply Lemma \ref{lem:mul} to a set of the $(j,i_j)$ elements of $W_2\circ \ReLU\circ W_1$ for each $\mathbf i=(i_1,i_2,\cdots,i_n)\in\mathbb Z_k^n$:
Then we have a network
$\operatorname{AppId}_m^n\in \mathcal F(L ,(n,3n(2k+1), n(2k+1),6n(2k+1)^n,\cdots,6n(2k+1)^n,(2k+1)^n))$
$(\,L=(m+5)\lceil \log_2(n(2k+1))\rceil +2 \,)$
such that for $\xx\in [0,1]^n$ and $\mathbf i\in\mathbb Z_k^n$
$$
\begin{array}{ll}
\displaystyle\max_{\mathbf i\in \mathbb Z_k^n}\left|\operatorname{AppId}_{m,\mathbf i}^n(\xx)-\prod_{j=1}^ng_1(x_j-i_j)\right|\le 3^n2^{-m}\quad\mbox{and}\quad \\
\operatorname{AppId}_m^n(\xx)=(\operatorname{AppId}_{m,\mathbf i}^n(\xx))_{\mathbf i\in\mathbb Z_k^n}.
\end{array}
$$

To make the network approximating $(f*g_n)_k$, let $W_{Last}:\RR^{(2k+1)^n}\to\RR$ be defined by
$$
W_{Last}(\xx)=\sum_{\mathbf i\in \mathbb Z_k^n}f(\mathbf i k^{-1})x_{\mathbf i}.
$$
Therefore, our network is
$
W_{Last}\circ\ReLU\circ \operatorname{AppId}_m^n  \in \mathcal F(L,(n,3n(2k+1), n(2k+1),6n(2k+1)^n,\cdots,6n(2k+1)^n,(2k+1)^n,1))
$
such that
$$
\begin{array}{ll}
&|W_{Last}\circ\ReLU\circ\operatorname{AppId}_m^n(\xx)-(f*g_n)_k|\\
&\displaystyle \le\left |\sum_{\mathbf i\in \mathbb Z_k^n}f(\mathbf i k^{-1})\operatorname{AppId}_{m,\mathbf i}^n(\xx)-\sum^{}_{\mathbf{i}\in\mathbb Z_k^n}f(\mathbf{i}k^{-1})g_{n}(k\xx-\mathbf i)\right|\\
&\displaystyle = \sum_{\mathbf i\in \mathbb Z_k^n}|f(\mathbf i k^{-1})||\operatorname{AppId}_{m,\mathbf i}^n(\xx)-g_{n}(k\xx-\mathbf i)|\\
&\displaystyle \le ||f||_\infty 3^n2^{-m}(2k+1).
\end{array}
$$
By Theorem \ref{thm:error}, we have
$$
\begin{array}{ll}
&|f(\xx)-W_{Last}\circ\ReLU\circ\operatorname{AppId}_m^n(\xx)|\\
&\le|f(\xx)-(f*g_n)_k(\xx)|
+|(f*g_n)_k(\xx)-W_{Last}\circ\ReLU\circ\operatorname{AppId}_m^n(\xx)|\\
&\le \displaystyle  \max_{j=1,2,\cdots,n}||\partial_{x_j}f(\xx)||_\infty n (2k)^{-1}+ ||f||_\infty3^n2^{-m}(2k+1).
\end{array}
$$

\appendix
\section{Proof of Theorem \ref{thm:error}}\label{sec:proof}
To prove Theorem \ref{thm:error}, we need the following Lemma:
\begin{lem}\label{lem:estimate}
For any $n\in\mathbb N$, we have
\begin{eqnarray}
\displaystyle\sum^{}_{\mathbf i\in\mathbb Z_k^n} g_{n}(k\xx-\mathbf i)=1\label{eq:I}\quad\mbox{and}\\
 \displaystyle\sum^{}_{\mathbf i\in\mathbb Z_k^n}\sum_{j=1}^n\left|( i_j k^{-1}-x_j)g_{n}(k\xx-\mathbf i)\right|\le n(2k)^{-1}.\label{eq:II}
 \end{eqnarray}
\end{lem}
\begin{proof}
We use the mathematical induction on $n$.
For $n=1$, we have
\begin{equation}\label{eq:inequal1}
\displaystyle\sum^{k}_{i=-k} g_{1}(kx-i)=g_{1}(kx-[kx])+g_{1}(kx-[kx]-1)=1 \qquad\mbox{and}
\end{equation}
\begin{equation}\label{eq:inequal}
\begin{array}{ll}
&\displaystyle \sum^{k}_{i=-k}|(x-ik^{-1})g_{1}(kx-i)|\\
&=|(x-[kx]k^{-1})| g_{1}(kx-[kx])+|(x-([kx]+1)k^{-1})| g_{1}(kx-[kx]-1)\\
&=k^{-1}|kx-[kx]| (1-kx+[kx])+k^{-1}|kx-[kx]-1| (kx-[kx])\\
&=2k^{-1}(kx-[kx])(1-(kx-[kx]))\le (2k)^{-1},
\end{array}
\end{equation}
since $0\le kx-[kx]\le 1$ and $x(1-x)\le 1/4$.
Suppose \eqref{eq:I} and \eqref{eq:II} hold for $n$.
Then we have for $\xx=(\tilde{\xx},x_{n+1})\in[-1,1]^{n}\times[-1,1]=[-1,1]^{n+1}$ and $\mathbf i=(\tilde{\mathbf i},i_{n+1})\in\mathbb Z_k^{n}\times\mathbb Z_k=\mathbb Z_k^{n+1}$,
$$
\begin{array}{ll}
\displaystyle\sum^{}_{\mathbf i\in\mathbb Z_k^{n+1}} g_{n+1}(k\xx-\mathbf i)\displaystyle=\sum^{}_{\mathbf i\in\mathbb Z_k^{n+1}} \prod_{j=1}^ng_1(kx_j-i_j)g_1(kx_{n+1}-i_{n+1})\\
\qquad\displaystyle=\sum^{}_{\tilde{\mathbf i}\in\mathbb Z_k^{n}} g_n(k\tilde{\xx}-\tilde{\mathbf i})(g_1(kx_{n+1}-[kx_{n+1}])+g_1(kx_{n+1}-[kx_{n+1}]-1))=1,
 \end{array}
$$
by \eqref{eq:inequal1}.
This gives \eqref{eq:I} for $n+1$.
Furthermore, we have
\begin{equation}\label{eq:IIes}
\begin{array}{ll}
\displaystyle\sum^{}_{\mathbf i\in\mathbb Z_k^{n+1}}\sum_{j=1}^{n+1}\left|( i_j k^{-1}-x_j)g_{n+1}(k\xx-\mathbf i)\right|\\
\displaystyle =\sum^{}_{\tilde{\mathbf i}\in\mathbb Z_k^{n}}\left(\sum_{j=1}^{n}\left|( i_j k^{-1}-x_j)\right| +|[kx_{n+1}]k^{-1}-x_{n+1}|\right)g_{n}(k\tilde{\xx}-\tilde{\mathbf i})g_{1}(kx_{n+1}- [kx_{n+1}])\\
\displaystyle \quad+\sum^{}_{\tilde{\mathbf i}\in\mathbb Z_k^{n}}\left(\sum_{j=1}^{n}\left|( i_j k^{-1}-x_j)\right| +|([kx_{n+1}]+1)k^{-1}-x_{n+1}|\right) g_{n}(k\tilde{\xx}-\tilde{\mathbf i})g_{1}(kx_{n+1}- [kx_{n+1}]-1)\\
\displaystyle =\sum^{}_{\tilde{\mathbf i}\in\mathbb Z_k^{n}}\sum_{j=1}^{n}\left|( i_j k^{-1}-x_j)\right|g_{n}(k\tilde{\xx}-\tilde{\mathbf i})  g_{1}(kx_{n+1}- [kx_{n+1}])\\
\displaystyle\quad +|[kx_{n+1}]k^{-1}-x_{n+1}| g_{1}(kx_{n+1}- [kx_{n+1}])\\
\displaystyle\quad +\sum^{}_{\tilde{\mathbf i}\in\mathbb Z_k^{n}}\sum_{j=1}^{n}\left|( i_j k^{-1}-x_j)\right| g_{n}(k\tilde{\xx}-\tilde{\mathbf i})g_{1}(kx_{n+1}- [kx_{n+1}]-1)\\
\displaystyle\quad +|([kx_{n+1}]+1)k^{-1}-x_{n+1}|g_{1}(kx_{n+1}- [kx_{n+1}]-1),
\end{array}
 \end{equation}
 where in the last equality, we used \eqref{eq:I}.
 By the induction hypothesis, \eqref{eq:IIes} becomes
 $$
 \begin{array}{ll}
 \displaystyle\sum^{}_{\mathbf i\in\mathbb Z_k^{n+1}}\sum_{j=1}^{n+1}\left|( i_j k^{-1}-x_j)g_{n+1}(k\xx-\mathbf i)\right|\\
 \displaystyle \le n(2k)^{-1}  (g_{1}(kx_{n+1}- [kx_{n+1}])+g_{1}(kx_{n+1}- [kx_{n+1}]-1))\\
\displaystyle\quad +(|[kx_{n+1}]k^{-1}-x_{n+1}| g_{1}(kx_{n+1}- [kx_{n+1}])\\
\displaystyle\quad+|([kx_{n+1}]+1)k^{-1}-x_{n+1}|g_{1}(kx_{n+1}- [kx_{n+1}]-1))\\
 \displaystyle \le  n(2k)^{-1} +(2k)^{-1},
\end{array}
 $$
 where in the last equality, we used \eqref{eq:inequal}.
\end{proof}

Now, we are ready to prove Theorem \ref{thm:error}.
For $\mathbf i\in\mathbb Z_k^n$ and $\xx\in[-1,1]^n$, we use the mean value theorem to get the representation
$$
f(\mathbf ik^{-1})=f(\xx)+\sum_{j=1}^n\partial_{x_j}f(\xx_{i,j})( i_j k^{-1}-x_j),
$$
where $\xx_{i,j}$ locates on the line with two end-points $\xx$ and $\mathbf ik^{-1} $. Using this expression, we find for $f(\xx)-(f*g_{n})_k(\xx)$ the representation of
$$
\begin{array}{l}
f(\xx)-(f*g_{n })_k(\xx)\\
\displaystyle=f(\xx)\left(1-\sum^{}_{\mathbf i\in\mathbb Z_k^n} g_{n}(k\xx-\mathbf i)\right)-\sum^{}_{\mathbf i\in\mathbb Z_k^n}\sum_{j=1}^n\partial_{x_j}f(\xx_{i,j})( i_j k^{-1}-x_j)g_{n}(k\xx-\mathbf i).
\end{array}
$$
Thus, we have
$$
\begin{array}{ll}
&|f(\xx)-(f*g_{n })_k(\xx)|\\
&\le \displaystyle||f||_\infty\left|\left(1-\sum^{}_{\mathbf i\in\mathbb Z_k^n} g_{n}(k\xx-\mathbf i)\right)\right|\displaystyle+\max_{j=1,2,\cdots,n}||\partial_{x_j}f(\xx)||_\infty \sum^{}_{\mathbf i\in\mathbb Z_k^n}\sum_{j=1}^n\left|( i_j k^{-1}-x_j)g_{n}(k\xx-\mathbf i)\right|.
\end{array}
$$
Lemma \ref{lem:estimate} gives
$$
|f(\xx)-(f*g_{n })_k(\xx)|\le \max_{j=1,2,\cdots,n}||\partial_{x_j}f(\xx)||_\infty n (2k)^{-1}.
$$


%
%


\bibliographystyle{plain}

\end{document}